\documentclass[letterpaper, 10 pt, conference]{ieeeconf}  

\IEEEoverridecommandlockouts                              

\usepackage{graphics} 
\usepackage{epsfig} 
\usepackage{times} 
\usepackage{amsmath,bm} 
\usepackage{amssymb}  
\usepackage{comment}
\usepackage{cite}
\usepackage{subcaption}
\usepackage{xcolor}
\usepackage{diagbox}
\usepackage[hyphens]{url}
\usepackage[colorlinks=true, breaklinks]{hyperref}

\newtheorem{theorem}{Theorem}
\newtheorem{problem}{Problem}
\newtheorem{assumption}{Assumption}
\newtheorem{definition}{Definition}
\newtheorem{remark}{Remark}
\newtheorem{example}{Example}

\allowdisplaybreaks

\title{\LARGE \bf
Scalable Multi-Robot Task Allocation and Coordination under Signal Temporal Logic Specifications
}

\author{Wenliang Liu, Nathalie Majcherczyk, and Federico Pecora
\thanks{All authors are with Amazon Robotics, North Reading, MA, USA 
        {\tt\small \{liuwll,majcherc,fpecora\}@amazon.com}}%
}

\begin{document}

\maketitle
\thispagestyle{empty}
\pagestyle{empty}

\begin{abstract}

Motion planning with simple objectives, such as collision-avoidance and goal-reaching, can be solved efficiently using modern planners. However, the complexity of the allowed tasks for these planners is limited. On the other hand, signal temporal logic (STL) can specify complex requirements, but STL-based motion planning and control algorithms often face scalability issues, especially in large multi-robot systems with complex dynamics. In this paper, we propose an algorithm that leverages the best of the two worlds.  We first use a single-robot motion planner to efficiently generate a set of alternative reference paths for each robot. Then coordination requirements are specified using STL, which is defined over the assignment of paths and robots' progress along those paths. We use a Mixed Integer Linear Program (MILP) to compute task assignments and robot progress targets over time such that the STL specification is satisfied. Finally, a local controller is used to track the target progress. Simulations demonstrate that our method can handle tasks with complex constraints and scales to large multi-robot teams and intricate task allocation scenarios.

\end{abstract}

\section{Introduction}

Temporal logics such as Linear Temporal Logic (LTL)~\cite{pnueli1977temporal} and Signal Temporal Logic (STL)~\cite{maler2004monitoring} provide a formal way to specify complex and time-related requirements for a given system, and have been widely used in robotics. 
In this paper, we focus on multi-robot systems subject to STL specifications. 
Deploying a fleet of autonomous robots to meet any form of specification or objective requires overcoming four challenges: (1)~determining how tasks should be distributed among robots (task assignment); (2)~deciding how robots should negotiate to collaboratively finish a task or use the shared resources (coordination); (3)~computing trajectories that robots should follow to complete their tasks (motion planning); (4)~ensuring that these trajectories are executed accurately respecting the robots' dynamics (control). 

Numerous STL control synthesis algorithms have been proposed 
\cite{raman2014model, sadraddini2015robust, pant2017smooth, haghighi2019control}. These methods aim to find a controller to make a system satisfy a given STL specification. Although called ``STL control synthesis", these methods tackle the motion planning and control problems 
simultaneously, i.e., they aim to find the state trajectory that satisfies the STL specification and the control inputs that realize this trajectory. When the system dynamics are complex and the planning horizon is large, these methods become computationally expensive, and cannot 
scale to large multi-robot systems. 

Most methods for coordinating robots using temporal logic~\cite{chen2011synthesis, schillinger2018simultaneous, kantaros2020stylus, sahin2019multirobot, leahy2021scalable, 9345973} discretize the state space to a graph or automaton, so they can decouple the control problem and focus on jointly solving the task allocation, coordination, and motion planning problems on the graph. However, an expressive enough abstraction of the state space can result in a very large graph, which again makes the problem computationally intractable. In~\cite{liu2023robust, liu2017communication}, continuous state space is considered and the control problem is solved jointly, which also scales badly. 
Another drawback of the above approaches is that the time steps needed might be too large for long-horizon planning. Unfortunately, STL specifications often span long horizons. 
A method based on time-stamped waypoints is proposed in~\cite{sun2022multi} which decouples control from motion planning without discretizing the state space, and is able to perform long-horizon planning with a relatively small number of time stamps.

All the above methods attempt to solve problems (1)--(4) 
simultaneously. Although this 
enlarges the search space, 
potentially improving the 
quality of the solution, computational complexity restricts scalability. 
In this paper, we propose a multi-robot task allocation and coordination algorithm that is decoupled from motion planning and control. This is motivated by the insights that most multi-robot tasks require each robot to move from the source location to the target location (or a sequence of target locations) while avoiding obstacles and to obey temporal and logical rules induced by their coordination. The former can be efficiently solved using heuristic- or sampling-based single-robot motion planning algorithms~\cite{hart1968formal,karaman2011sampling,kavraki1996probabilistic}. 
However, these planners cannot enforce the latter, which is where STL is truly needed.

In this paper, we first use single-robot motion planning 
to efficiently generate a set of alternative reference paths for each robot. Then we use STL to define coordination specifications over the progress of robots along their paths. In this way, 
each robot only 
decides which path to take and how fast to track the path, without considering the entire state space or its dynamics. A similar idea is introduced in~\cite{forte2021online,mannucci2021provably}, which only consider simple constraints rather than STL specifications. Similar to~\cite{sun2022multi}, we search for a sequence of time-stamped target progress points for the robots to track along their paths, which can be solved using Mixed Integer Linear Program (MILP). Finally, the target progress points on a reference path are tracked using a local controller of each robot, which guarantees the satisfaction of the STL specification. Our approach breaks down the STL-based coordination problem into three parts: task allocation and coordination subject to STL specifications; single-robot motion planning; and single-robot control. Although this approach slightly reduces the search space of the overall problem due to the decoupling, it significantly reduces the computational cost. We find that in practice, it is enough to realize many interesting multi-robot use cases. 

The contributions of this paper are threefold. (1)~We propose a scalable algorithm to operate fleets of robots subject to STL specifications by decoupling motion planning and control from the task allocation and coordination problems. (2)~We prove formally 
that our algorithm 
satisfies the STL specification. (3)~We evaluate our algorithm on a variety of realistic applications, showing that it outperforms state-of-the-art methods in terms of running time and is able to scale to large teams and complex STL specifications.

\section{Problem Formulation}
In this paper, we use $\mathbb R$ and $\mathbb B$ to denote real values and binary values, respectively. For a vector $x\in\mathbb R^n$, let $B_l(x,\epsilon):=\{y\in\mathbb R^n\ |\ \|y-x\|_l<\epsilon\}$ be the $\epsilon$-ball centered at $x$, where $\|\cdot\|_l$ denotes the $l$-norm.

\subsection{System Model}
Consider a set of $N$ robots $\{r_1,\ldots,r_N\}$ sharing an obstacle-free space $\mathcal W_{\mathrm{free}}\subseteq\mathbb R^3$. Each robot $r_i$ is defined as a tuple $\langle Q_i, q_i^0, R_i, \{p_i^j\}_{j=1}^{M_i}\rangle$ where:
\begin{itemize}
    \item $Q_i$ is the space of obstacle-free configurations of $r_i$;
    \item $q_i^0\in Q_i$ is the initial configuration of $r_i$;
    \item $R_i:Q_i\rightarrow 2^{\mathcal W_{\mathrm{free}}}$ maps the configuration of $r_i$ to a geometry describing the space occupied by $r_i$;
    \item $\{p_i^j\}_{j=1}^{M_i}$ is a set of $M_i$ reference paths assigned to $r_i$, where a reference path $p_i^j:[0,g_i^j]\rightarrow Q_i$ maps  progress between $0$ and $g_i^j$ to a configuration, $g_i^j$ is the maximum progress corresponding to the goal configuration which is proportional to the path length. Here, $\{p_i^j\}_{j=1}^{M_i}$ corresponds to potential tasks assigned to $r_i$.
\end{itemize}

We assign a vector $\mathbf z_i\in\mathbb B^{M_i}$ consisting of $M_i$ binary variables $\mathbf z_i = [z_i^1,\ldots,z^{M_i}_i]^\top$ to each robot $r_i$, indicating which reference path is selected. Specifically, $z_i^j=1$ indicates $p_i^j$ is selected. Since each robot can follow one and only one reference path, we have the following constraint:
\begin{equation}
\label{eq:milp-ta}
    \sum_{j=1}^{M_i} z_i^j = 1, \ i=1,\ldots,N.
\end{equation}
Let the joint selection vector be $\mathbf z = [\mathbf z_1^\top,\ldots,\mathbf z_N^\top]^\top$, which is the concatenation of all selection vectors for all robots.

The temporal profile of $r_i$ is $\sigma_i:[0,T]\rightarrow R_{\geq0}$, which maps time to the progress on its selected path, $T$ is the time bound. We assume $\sigma_i(t)$ is monotonically increasing. We denote the joint temporal profile as $\bm\sigma:[0,T]\rightarrow\mathbb R^N_{\geq0}$, which maps time to the joint progress of all robots. For simplicity, we occasionally omit the variable $t$ and directly use $\sigma_i\in\mathbb R_{\geq0}$ to denote the progress when the context makes it clear. 

\subsection{Reference Path STL}
In this paper, we define a fragment of STL over $\mathbf s = [\mathbf z, \bm\sigma]$, which is the concatenation of the joint selection vector $\mathbf z$ and the joint temporal profile $\bm\sigma$. In the following, we refer to this STL fragment as reference-path STL (RP-STL).

Different from the general case of STL, the predicate of RP-STL $\mu_i$ is restricted to a specific robot $r_i$ in the form of:
$ \sigma_i(t) - \mathbf b_i^\top\mathbf z_i\geq 0$,
where $\mathbf b_i=[b_i^1,\ldots,b_i^{M_i}]^\top\in\mathbb R^{M_i}$. Depending on the value of $\mathbf b_i$, the predicate $\mu_i$ can have different interpretations. For example, if $b_i^1 = 10$ and $b_i^j = M$ $\forall j\neq 1$, where $M$ is a large value, then $\mu_i$ is evaluated as \emph{true} at time $t$ if and only if the path $p_i^1$ is selected and $\sigma_i(t) \geq 10$. On the contrary, if $b_i^j = -M$, $\forall j\neq 1$, then $\mu_i$ is automatically satisfied if the path $p_i^1$ is not selected. It is only evaluated as \emph{false} if $p_i^1$ is selected but $\sigma_i(t) < 10$. 

We also define a counting formula for RP-STL as a tuple $(\{\varphi_l\}_{l=1}^L, m)$ where $\{\varphi_l\}_{l=1}^L$ is a set of $L$ RP-STL formulas and $m\leq L$ is a positive integer. A counting formula is evaluated as \emph{true} if and only if there are at least $m$ subformulas $\varphi_l$ evaluated as \emph{true}. Although a counting formula can be translated into standard STL using combinatorially many disjunctions, our definition provides a concise way to formulate this kind of requirements, which is very useful in practice. In addition, in Sec.~\ref{sec:milp} we propose a MILP encoding for counting formulas, which avoids introducing combinatorially many binary variables.

We recursively define the \emph{syntax} of RP-STL as:
\begin{equation}
    \label{eq:stl-syntax}
    \begin{aligned}
        \varphi = &\mu_i\ |\ \neg\mu_i \ |\ (\{\varphi_l\}_{l=1}^L, m)\ |\ \varphi_1\land\varphi_2 \ |\ \varphi_1\lor\varphi_2 \\
        &|\ F_{[a,b]}\varphi \ |\ G_{[a,b]}\varphi \ |\ \varphi_1 U_{[a,b]} \varphi_2,
    \end{aligned}
\end{equation}
where $\varphi$, $\varphi_1$, $\varphi_2$, $\varphi_l$ are RP-STL formulas, $\neg$, $\land$, $\lor$ are the \emph{negation}, \emph{conjunction}, and \emph{disjunction}, $F_{[a,b]}$, $G_{[a,b]}$, $U_{[a,b]}$ are the temporal operators \emph{eventually}, \emph{always}, and \emph{until}. 

The fact that a signal $\bm s$ satisfies an RP-STL formula $\varphi$ at time $t$ is denoted as $(\bm s,t) \models \varphi$. Intuitively, $(\bm s,t)\models F_{[a,b]}\varphi$ states that $\varphi$ must become true at some time point in $[t+a,t+b]$, $(\bm s,t)\models G_{[a,b]}\varphi$ means that $\varphi$ must be true at all time points in $[t+a,t+b]$, and $(\bm s,t)\models \varphi_1 U_{[a,b]}\varphi_2$ requires that $\varphi_2$ becomes true at some time in $[t+a,t+b]$ and $\varphi_1$ is true at all time before that. The satisfaction of the counting formula $(\bm s,t)\models (\{\varphi_l\}_{l=1}^L, m)$ is defined as $\sum\nolimits_{l=1}^L \mathbf1((\bm s,t)\models\varphi_l) \geq m$, where $\mathbf 1(\text{true})=1$ and $\mathbf 1(\text{false})=0$. The operators and counting formula can be arbitrarily nested to express more complex requirements. For simplicity, we will omit $t$ when $t=0$ and denote $(\bm s,0)\models\varphi$ as $\bm s\models\varphi$.

\begin{remark}
\label{rm:neg}
    In syntax \eqref{eq:stl-syntax}, negation can only be applied to predicates, known as the Negation Normal Form (NNF). This is not restrictive, as any STL formula can be put into NNF~\cite{baier2008principles}.  
    For the counting formula, $\neg(\{\varphi_l\}_{l=1}^L, m)$ is equivalent to $(\{\neg\varphi_l\}_{l=1}^L, L - m + 1)$. Hence, any RP-STL can be put into NNF.
    By applying negation to a counting formula, i.e., $\neg(\{\varphi_l\}_{l=1}^L,m)$, we can require that less than $m$ subformulas $\varphi_l$ 
    are \emph{true}.
\end{remark}

\subsection{Interference Constraints}

One important kind of coordination constraints that can be expressed by RP-STL is the interference constraints, i.e., how robots traverse the shared space without collision. Consider two robots $r_i$ and $r_{i'}$, and two paths $p_i^j$ and $p_{i'}^{j'}$ as in Fig.~\ref{fig:example} (left). The two paths have one critical section defined as a pair of intervals $([l,u], [l',u'])$ which satisfies (1) $\forall \sigma_i\in[l,u]$, $\exists \sigma_{i'}\in[l',u']$ such that $R_i(p_i^j(\sigma_i))\cap R_{i'}(p_{i'}^{j'}(\sigma_{i'})) \neq \emptyset$, and vice versa; (2) the intervals $[l,u]$ and $[l',u']$ are maximal, i.e., there is no way to grow them and still satisfy the first requirement. A collision is only possible if both robots are in the critical section. The interference constraint requires that when one robot is in the critical section, the other robot should not enter it. Formally, this can be written as:
    \begin{equation}
    \label{eq:interfere}
    \begin{aligned}
        \varphi_{ii'}^{jj'}  = &(\sigma_{i'} < [M  \cdots   \overset{j'\text{-th}}{l'}  \cdots  M]\mathbf z_{i'}^\top\ U_{[0,T]}\\ 
        &  \hspace{+10pt} \sigma_i \geq [-M  \cdots   \overset{j\text{-th}}u  \cdots  -M]\mathbf z_i^\top) \\
        &\lor  (\sigma_i < [M  \cdots   \overset{j\text{-th}}l  \cdots  M]\mathbf z_{i}^\top\ U_{[0,T]}\\ 
        & \hspace{+10pt} \sigma_{i'} \geq [-M  \cdots   \overset{j'\text{-th}}{u'}  \cdots  -M]\mathbf z_{i'}^\top)).
    \end{aligned}
    \end{equation}
In English, $\varphi_{ii'}^{jj'}$ means that if paths $p_i^j$ and $p_{i'}^{j'}$ are selected, then $r_{i'}$ cannot enter the critical section \emph{until} $r_i$ leaves it, \emph{or} $r_i$ cannot enter the critical section \emph{until} $r_{i'}$ leaves it. If any one of $p_i^j$ and $p_{i'}^{j'}$ is not selected, then $\varphi_{ii'}^{jj'}$ is 
satisfied.

We add the interference constraints $\varphi_{ii'}^{jj'}$ for all critical sections in all path pairs $(p_i^j,p_{i'}^{j'})$ where $i\neq i'$. These constraints avoid any collisions between robots.



\begin{remark}
The above STL formula requires exclusive use of the critical section, which might be too conservative. 
We can partially rely on the robot's own autonomy (the local controller) and relax this constraint by replacing $u$ and $u'$ in \eqref{eq:interfere} with $\delta$ and $\delta'$ where
\begin{equation*}
    \begin{aligned}
        &\delta = \inf\{\sigma\in[l,u]\ |\ \forall d>\sigma,R_i(p_i^j(d))\cap R_{i'}(p_{i'}^{j'}(l')) = \emptyset\},\\
        &\delta' = \inf\{\sigma\in[l',u']\hspace{+1pt}|\hspace{+1pt}\forall d>\sigma,R_{i'}(p_{i'}^{j'}(d))\cap R_{i}(p_{i}^{j}(l)) = \emptyset\}.
    \end{aligned}
\end{equation*}
Here, $\delta$ is the smallest progress for $r_i$ such that it will not block $r_{i'}$ from entering the critical section. 
\end{remark}

\begin{example}
\label{ex:bridge}
    Consider a team of robots $\mathcal R = \{r_1,r_2,r_3\}$ in an environment where a bridge (occupying the space $\mathcal B$) goes across a river, as shown in Fig.~\ref{fig:example} (right). Robot $r_1$ is assigned two reference paths $p_1^1$ and $p_1^2$, and $r_2$ and $r_3$ are assigned $p_2^1$ and $p_3^1$, respectively. There are totally $3$ critical sections among these $4$ paths. In addition to the interference constraints, we require that no more than $2$ robots can be on the bridge at the same time due to the weight limit of the bridge. Let $[l_i^j,u_i^j]$ be an interval on the path $p_i^j$ such that $\sigma\in [l_i^j,u_i^j]$ if and only if $R_i(p_i^j(\sigma))\cap\mathcal B\neq\emptyset$. These requirements can be written as an RP-STL formula:
    \begin{equation}
    \label{eq:example}
    \begin{aligned}
        \varphi = & \varphi_{12}^{11} \land \varphi_{12}^{21} \land \varphi_{13}^{11} \land G_{[0,T]}\\
        \neg\bigg(\Big\{&(\sigma_1 \geq [l_1^1\ M]\begin{bmatrix}z_1^1 \\ z_1^2 \end{bmatrix}) \land (\sigma_1 < [u_1^1\ -M]\begin{bmatrix}z_1^1 \\ z_1^2 \end{bmatrix}), \\
        &(\sigma_1 \geq [M\ l_1^2]\begin{bmatrix}z_1^1 \\ z_1^2 \end{bmatrix}) \land (\sigma_1 < [-M\ u_1^2]\begin{bmatrix}z_1^1 \\ z_1^2 \end{bmatrix}), \\
        &  (\sigma_2 \geq l_2^1 \land \sigma_2 < u_2^1),\ (\sigma_3 \geq l_3^1 \land \sigma_3 < u_3^1)\Big\},3\bigg).
    \end{aligned}
    \end{equation}
    Here we omit $z_2^1$ and $z_3^1$ as they are always equal to $1$.
    The counting formula includes $4$ subformulas corresponding to $4$ paths (and $4$ intervals $[l_i^j,u_i^j]$). The (negative) counting constraint requires that the progress along less than $3$ paths can be in the interval $[l_i^j,u_i^j]$ at the same time. If a path is not selected, the corresponding subformula is automatically violated (not counted as on the bridge) regardless of progress. 
\end{example}

\begin{figure}
    \centering
    \begin{subfigure}[b]{0.19\textwidth}
        \centering
        \includegraphics[width=\textwidth]{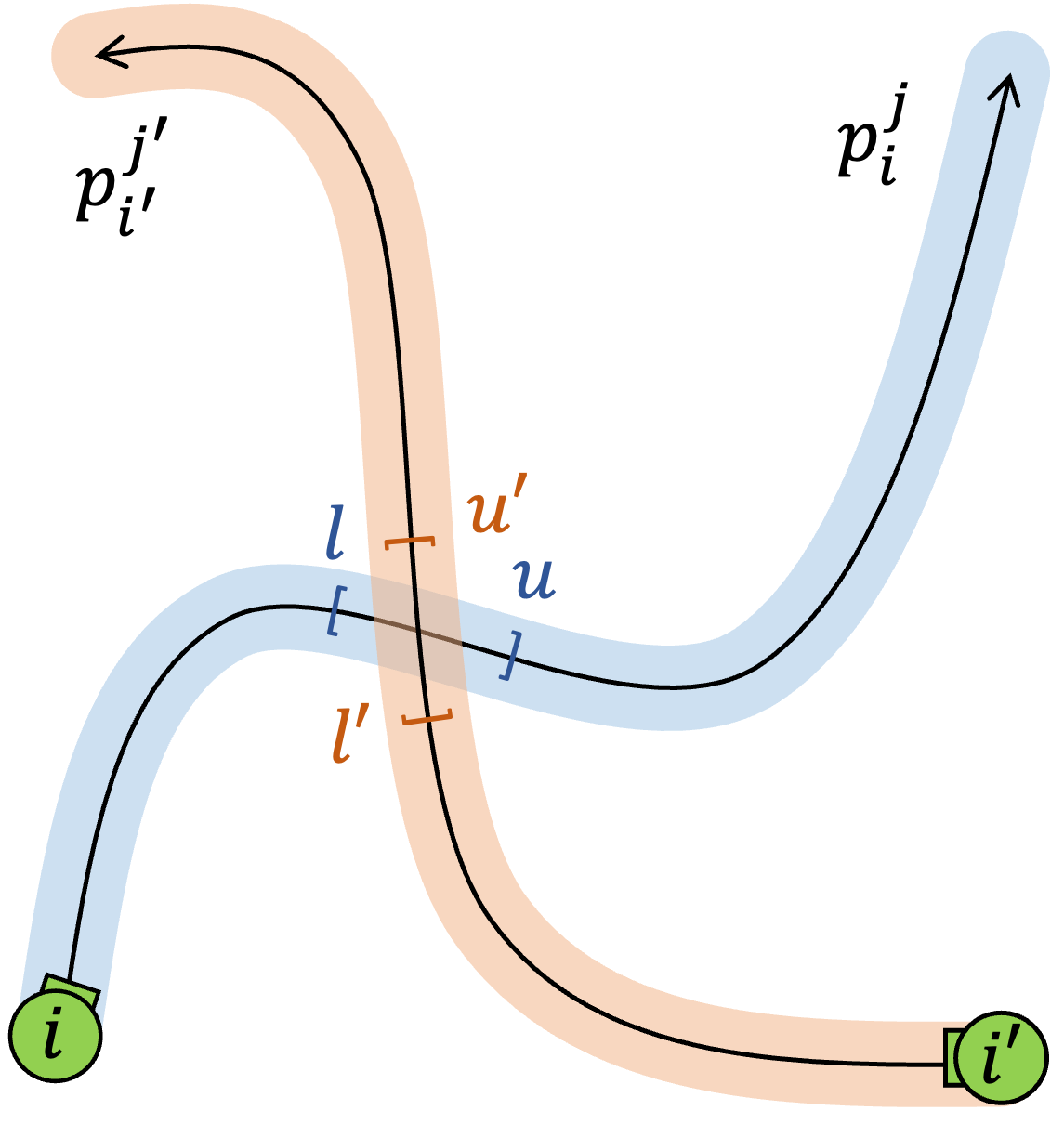}
    \end{subfigure}
    \quad
    \begin{subfigure}[b]{0.24\textwidth}
        \centering
        \includegraphics[width=\textwidth]{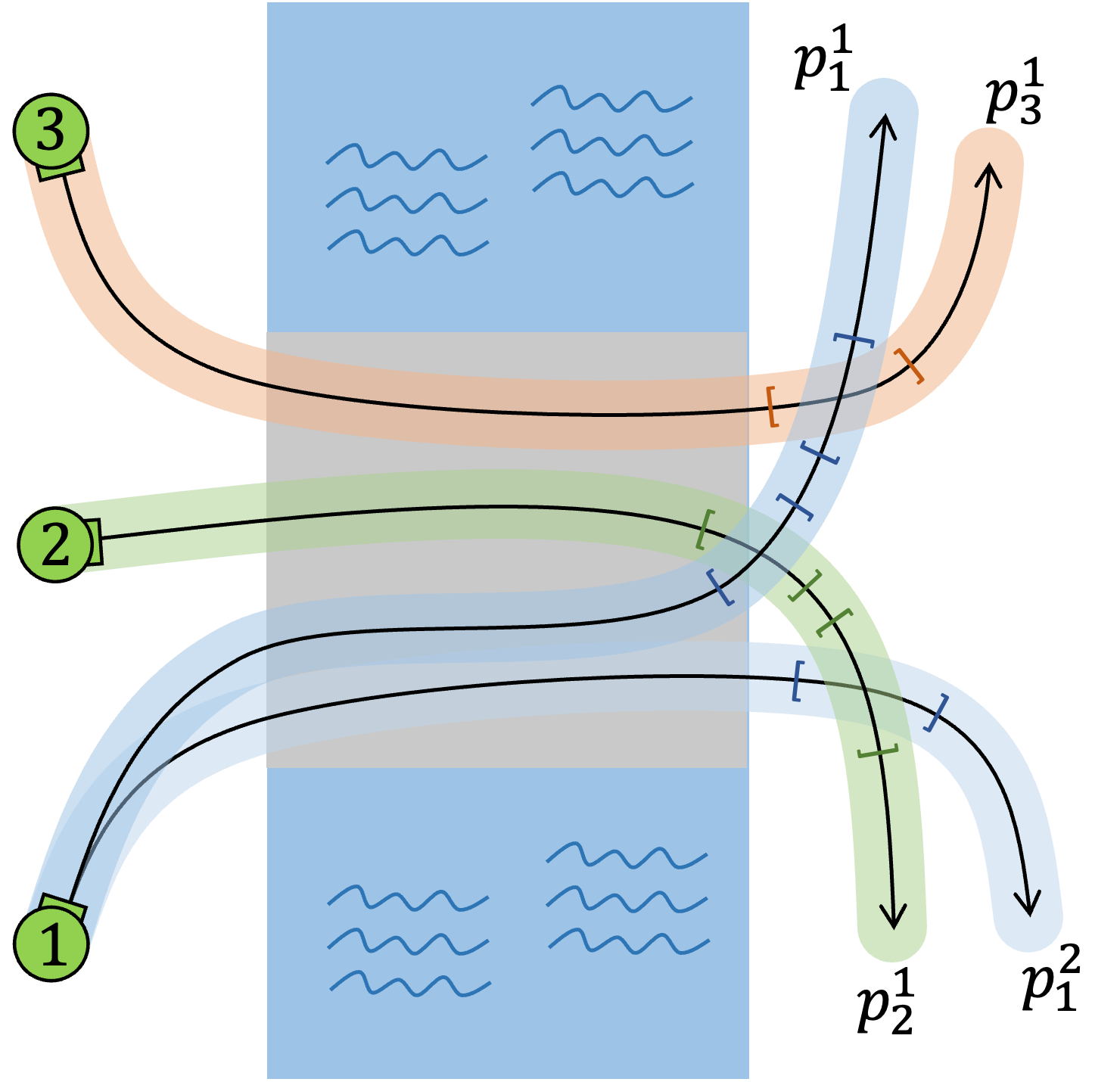}
    \end{subfigure}
    \caption{\small Left: The critical section between two robots $r_i$ and $r_{i'}$. Right: $3$ robots crossing a bridge (gray) on a river (blue).}
    \label{fig:example}
    \vspace{-8pt}
\end{figure}

\subsection{The Task Allocation and Coordination Problem}
Our goal is to find the joint selection vector $\mathbf z$ and the joint temporal profile $\bm\sigma$ that satisfy a given RP-STL specification $\varphi$ and minimizes a given cost function. Let a time-stamped joint progress (TSJP) be a pair $(t_k, \bm \sigma_{(k)})$, where $k \in \{0,1,\ldots,K\}$,  $t_k\in[0,T]$ is a time stamp and $\bm \sigma_{(k)} = [\sigma_{1,(k)},\ldots, \sigma_{N,(k)}]\in\mathbb R_{\geq0}^N$ is a corresponding joint progress. Note that we use bold font and parentheses on $k$ to distinguish $\bm\sigma_{(k)}$ from $\sigma_i$, the latter being the temporal profile of $r_i$. Instead of searching for an accurate joint temporal profile over all possible functions over time, we search for a set of joint temporal profiles constructed from a sequence of TSJPs $\{(t_k, \bm \sigma_{(k)})\}_{k=0}^K$, where $t_k$ is the $k$-th time stamp and $\bm \sigma_{(k)}$ is the $k$-th joint progress, $0=t_0<t_1<\ldots<t_K\leq T$, $\bm 0=\bm \sigma_{(0)}\leq\bm \sigma_{(1)}\leq\ldots\leq\bm \sigma_{(K)}=[[g_1^1,\ldots,g_1^{M_1}]\cdot\mathbf z_1, \ldots, [g_N^1,\ldots,g_N^{M_N}]\cdot\mathbf z_N]$. Note that the equality and inequality signs for $\bm \sigma_{(k)}$ are interpreted element-wise. Specifically, given a sequence of TSJPs and an $\epsilon>0$, we construct a set of joint temporal profiles as
\vspace{-3pt}
\begin{equation}
    \label{eq:S}
    \begin{aligned}
        \mathcal S_\epsilon\Big(\{(t_k, \bm \sigma_{(k)})\}_{k=0}^K\Big) &:= \\
        \{\bm\sigma\ |\ \bm\sigma(t_k) &\in B_\infty(\bm \sigma_{(k)},\epsilon),\ k=0,\ldots,K\}.
    \end{aligned}
\end{equation}
Now, our goal becomes finding a joint selection vector $\mathbf z$ and a sequence of TSJPs $\{(t_k, \bm \sigma_{(k)})\}_{k=0}^K$ such that $\forall \bm\sigma\in\mathcal S_\epsilon\big(\{(t_k, \bm \sigma_{(k)})\}_{k=0}^K\big)$, $[\mathbf z,\bm\sigma]\models\varphi$, and a cost is minimized.

Since the robots cannot make progress arbitrarily fast, we add the following constraint on the TSJPs for all robots $r_i$:
\begin{equation}
    \label{eq:vmax}
        |\sigma_{i,(k+1)} - \sigma_{i,(k)}|\leq v^{max}_{i}\cdot(t_{k+1}-t_k), k=0,\ldots,K-1,
\end{equation}
where $v^{max}_{i}\in\mathbb R_{\geq0}$ is the maximum speed 
for robot $i$, which is assumed to be a constant for different paths of a same robot, but can be different for different robots.
Now we formally state the task allocation and coordination problem.
\begin{problem}
\label{pb}
    Consider a set of $N$ robots. Given a set of reference paths $\{p_i^j\}_{j=1}^{M_i}$ for each robot $r_i$ and an RP-STL specification $\varphi$, find the joint selection vector $\mathbf z$ and a sequence of TSJPs $\{(t_k, \bm \sigma_{(k)})\}_{k=0}^K$ such that $\forall \bm\sigma\in\mathcal S_\epsilon\big(\{(t_k, \bm \sigma_{(k)})\}_{k=0}^K\big)$, $[\mathbf z,\bm\sigma]\models\varphi$, and a cost $\mathcal L$ is minimized:
\begin{equation}
    \label{eq:optimization}
    \begin{aligned}
        &\min_{\mathbf z, \{(t_k, \bm \sigma_{(k)})\}_{k=0}^K}\ \mathcal L(\mathbf z, \{(t_k, \bm \sigma_{(k)})\}_{k=0}^K) \\
        &\hspace{+10pt}\text{s.t. } [\mathbf z,\bm \sigma]\models \varphi,\ \forall\bm\sigma\in\mathcal S_\epsilon\big(\{(t_k, \bm \sigma_{(k)})\}_{k=0}^K\big),\\
        &\hspace{+10pt}\{(t_k, \bm \sigma_{(k)})\}_{k=0}^K\ \text{satisfies}\ \eqref{eq:vmax}, \mathbf z \ \text{satisfies}\  \eqref{eq:milp-ta}.
    \end{aligned}
    \vspace{+5pt}
\end{equation}
\end{problem}

Some examples of the cost function include the makespan $t_K$ or the sum of travel time $\sum_{i=1}^N T_i$ where
\begin{equation}
\label{eq:T_i}
    T_i = \inf \{t_k\ |\ \sigma_{i,(k)}=[g_i^1,\ldots,g_i^{M_i}]\cdot\mathbf z_i\}.
\end{equation}

\section{MILP-Based Solution}
\label{sec:milp}
We solve Problem~\ref{pb} by encoding \eqref{eq:optimization} into a Mixed-Integer Linear Program (MILP), which can be solved efficiently using off-the-shelf solvers such as Gurobi~\cite{gurobi}. 

To encode all constraints in \eqref{eq:optimization} as mixed-integers linear constraints, we first encode these constraints into a Linear and Counting Constraints Formula (LCCF), which is a logic sentence of atomic formulas connected by conjunctions, disjunctions, and counting operators. Each atomic formula is in the form of $\texttt{LE}\geq0$, where $\texttt{LE}$ is a linear expression of the continuous and binary variables, including the TSJPs and the selection vector. Second, we eliminate all the disjunctions and counting operators in the LCCF, which makes the LCCF a conjunction of mixed integer linear constraints. Finally, we add all these constraints into the MILP solver to find the solution. The encoding is inspired by \cite{sun2022multi}, which searches for a piece-wise linear path connecting time-stamped waypoints in the state space to satisfy STL specifications. Different from \cite{sun2022multi}, here we do not require linearity on the segments between TSJPs. We only assume the temporal profile for each robot is monotonically increasing.

\subsection{Constructing the LCCF}
\label{sec:lccf}
\noindent\paragraph{Selection vector encoding}
Since constraints \eqref{eq:milp-ta} are already in the form of $\texttt{LE}\geq0$, we just conjunct them together to form a single LCCF.

\paragraph{RP-STL encoding}
We inductively construct an LCCF that encodes the RP-STL formula $\varphi$ in \eqref{eq:optimization} by constructing an LCCF $z_k^{\varphi}$ for each segment between two adjacent TSJPs such that $z_k^{\varphi}$ is true if and only if $\forall t\in[t_k,t_{k+1}]$ and 
$\forall \bm\sigma\in\mathcal S_\epsilon\big(\{(t_k, \bm \sigma_{(k)})\}_{k=0}^K\big)$, $([\mathbf z,\bm \sigma],t)\models \varphi$. We refer to this property as the soundness property. Then $z_0^\varphi$ will be the LCCF we want. Specifically, $z_k^\varphi$ is recursively encoded as:
\begin{align}
\label{eq:z-mu}
    &z_k^{\mu_i} = (\sigma_{i,(k)}-\bm b_i^\top\bm z_i \geq \epsilon) \land (\sigma_{i,(k+1)}-\bm b_i^\top\bm z_i \geq\epsilon),\\
    \label{eq:z-cnt}
    &z_k^{(\{\varphi_l\}_{l=1}^L, m)} = (\sum_{l=1}^{L}\mathbf 1(z_k^{\varphi_l})\geq m),\\
    & \begin{aligned}
    z_k^{G_{[a,b]}\varphi} &= \bigwedge_{l=0}^{K-1} ([t_l,t_{l+1}]\\&\cap
    [t_k+a-\epsilon_t,t_{k+1}+b]\neq\emptyset\Rightarrow z_l^\varphi),
    \label{eq:z-always}\end{aligned}\\
    &\begin{aligned}z_k^{F_{[a,b]}\varphi} & = (t_{k+1}-t_k\leq b-a-\epsilon_t)\land \bigvee_{l=0}^{K-1} ([t_l,t_{l+1}]\\&\cap[t_{k+1}+a,t_k+b-\epsilon_t]\neq\emptyset\land z_l^\varphi),\label{eq:z-eventually}\end{aligned}
\end{align}
where $\epsilon_t$ is a positive value that adds robustness to the timimg, $\Rightarrow$ is the \emph{implication} operator defined as $\varphi_1\Rightarrow\varphi_2$ iff $\neg\varphi_1\lor\varphi_2$.
The construction laws for \emph{Until} (as a combination of \eqref{eq:z-always} and \eqref{eq:z-eventually}) and Boolean operators are omited and as reported in~\cite{sun2022multi}. 
By using 
these and \eqref{eq:z-mu}-\eqref{eq:z-eventually}, we can recursively construct $z_0^\varphi$ that encodes the entire formula $\varphi$.


We prove the aforementioned soundness property by induction. We start from the predicate $\mu_i$. Intuitively, if $z_k^{\mu_i}$ defined in \eqref{eq:z-mu} is true, then any $\sigma_i$ that satisfies $\sigma_i(t_k)\in B_\infty(\sigma_{i,(k)},\epsilon)$ and $\sigma_i(t_{k+1})\in B_\infty(\sigma_{i,(k+1)},\epsilon)$ also satisfies $\mu_i$ at $t_k$ and $t_{k+1}$. Since $\sigma_i(t)$ is monotonically increasing, $\mu_i$ is satisfied at all $t\in[t_k,t_{k+1}]$. The encoding for the counting formula \eqref{eq:z-cnt} is sound by definition. The proof for other Boolean and temporal operators is similar as in~\cite{sun2022multi}. Hence, the LCCF constructed above is sound. 

\paragraph{Sum of travel time encoding}
When using makespan $t_K$ as the cost function, we do not need further encoding. 
For sum of travel time, we introduce additional variables $T_i$ with the following constraints to enforce~\eqref{eq:T_i}:
\begin{equation}
\begin{aligned}
    &(\sigma_{i,(k)}\leq [g_i^1,\ldots,g_i^{M_i}]^\top\mathbf z_i) \land (T_i \geq t_{k+1})\\
    \lor & (\sigma_{i,(k)}\geq [g_i^1,\ldots,g_i^{M_i}]^\top\mathbf z_i) \land (T_i \leq t_{k}).
\end{aligned}
\end{equation}

\subsection{Eliminating Disjunctions and Counting Operators}
Finally, we eliminate disjunction and counting formulas using the Big M method. This introduces new binary variables proportional to the number of disjunctions and subformulas in counting formulas but does not increase the number of constraints. Details are omitted here for brevity.

\subsection{Overall MILP Approach}
Now we have transformed the constraints in \eqref{eq:optimization} into linear constraints of the continuous and binary variables, which makes \eqref{eq:optimization} a MILP problem. Solving this MILP gives us the optimal solution. The computational cost of solving a MILP depends on the number of binary variables, which is $O(K^2\cdot|\varphi| + N\cdot M_i)$ in our encoding, where $|\varphi|$ is the number of operators in $\varphi$, and $M_i$ is the number of paths per robot.

\section{Local Controllers}
We use a local controller for each robot to track its selected reference path according to the sequence of TSJPs obtained by solving Problem~\ref{pb}. We make the following assumption on the local controller. Relaxation of it is discussed later. 

\begin{assumption}
    \label{as:in-ball}
    Consider a set of $N$ robots which have reached the joint progress $\bm \sigma_{(k)}$ at time $t_k$ within a distance of $\epsilon$, i.e., $\bm\sigma(t_k) \in B_\infty(\bm \sigma_{(k)},\epsilon)$. Given a target TSJP $(t_{k+1},\bm \sigma_{(k+1)})$ that satisfies \eqref{eq:vmax}, the local controllers can always make robots reach the joint progress $\bm \sigma_{(k+1)}$ within the distance $\epsilon$ at $t_{k+1}$, i.e., $\bm\sigma(t_{k+1}) \in B_\infty(\bm \sigma_{(k+1)},\epsilon)$.
\end{assumption}
\begin{theorem}
\label{thm:liveness}
    Under Assumption~\ref{as:in-ball}, given an RP-STL specification $\varphi$ and a sequence of TSJPs from solving \eqref{eq:optimization}, the joint temporal profile $\bm\sigma$ executed by the local controller that tracks these TSJPs always satisfies $\varphi$, i.e., $[\mathbf z, \bm\sigma]\models\varphi$.
\end{theorem}
\begin{proof}
    Since $\bm\sigma(0) = \bm\sigma_{(0)}$, the above theorem can be proved through induction based on Assumption~\ref{as:in-ball}.
\end{proof}

Assumption~\ref{as:in-ball} is not restrictive as many existing techniques can provide us such a local controller, e.g., the control Lyapunov functions \cite{rodriguez2014trajectory}. However, in practice it is possible that 
robots make tracking deviations $> \epsilon$, e.g., robots stalling or significantly reducing their 
speed. In such cases, it is hard to guarantee the satisfaction of the specification. 
However, under the following relaxed assumption, we can ensure that anomalies are detectable before the specification is violated: 

\begin{assumption}
    \label{as:non-exceed}
    Consider a set of $N$ robots which reaches the joint progress $\bm \sigma_{(k)}$ at time $t_k$ within a distance of $\epsilon$, i.e., $\bm\sigma(t_k) \in B_\infty(\bm \sigma_{(k)},\epsilon)$. Given a target TSJP $(t_{k+1},\bm \sigma_{(k+1)})$ that satisfies \eqref{eq:vmax}, the local controllers never make robots exceed the joint progress $\bm \sigma_{(k+1)}$ by $\epsilon$ before $t_{k+1}$, i.e., $\sigma_i(t_{k+1}) < \sigma_{i,(k+1)} + \epsilon$, $\forall i=1,\ldots,N$.
\end{assumption}
\begin{theorem}
\label{thm:safety}
    Under Assumption~\ref{as:non-exceed}, given an RP-STL specification $\varphi$ and a sequence of TSJPs from solving \eqref{eq:optimization}, if the joint progress made by the local controllers satisfies $\bm\sigma(t_k) \in B_\infty(\bm \sigma_{(k)},\epsilon)$, then $\varphi$ is not violated before $t_{k+1}$.
\end{theorem}
\begin{proof}
    (sketch) To prove this, we need to show that when the TSJP is violated for the first time, the specification $\varphi$ is still not violated. Assumption~\ref{as:non-exceed} ensures that $\neg\mu_i$, i.e., $\sigma_i(t)<\mathbf b_i^\top\mathbf z_i$, is never violated. So the violation can only come from  $\sigma_i(t)\geq\mathbf b_i^\top\mathbf z_i$. For $G_{[a,b]}\varphi$, \eqref{eq:z-always} requires satisfaction of $\varphi$ over an interval containing $[a-\epsilon_t,b]$. Since $\bm\sigma$ is monotonically increasing, once $\varphi$ is satisfied, it is satisfied all the time later. So the first violation can only happen at the beginning of the interval, which is earlier than $a$. Hence, $G_{[a,b]}\varphi$ is still not violated. Similarly, $F_{[a,b]}\varphi$ needs to be satisfied over an interval with non-empty intersection with $[a,b-\epsilon_t]$. The first violation can only happen at the beginning of the interval, which is earlier than $b$, so $F_{[a,b]}\varphi$ is still not violated. Same for $\varphi_1U_{[a,b]}\varphi_2$, which can be seen as a combination of \emph{always} and \emph{eventually}. 
\end{proof}

In practice, we can check if every robot reaches the TSJP at each time stamp $t_k$. If the deviation is less than $\epsilon$, then we continue to track the current sequence of TSJPs, ensuring that $\varphi$ is not violated before $t_{k+1}$. If the deviation exceeds $\epsilon$ due to unforeseen factors, we replan taking these factors into account. Replanning is beyond the scope of this paper and will be explored in future work.

\section{Experimental Evaluation}
In this section, we evaluate the proposed approach on several benchmark scenarios and compare it with other methods. All experiments were run in the ARGoS simulator~\cite{Pinciroli:SI2012} on a Mac computer with M3 Pro CPU and 18 GB RAM. We use the ARA* algorithm~\cite{likhachev2003ara} in SBPL~\cite{likhachev2010search} for motion planning. 

\subsection{Scenarios}
\begin{figure*}
    \centering
    \begin{subfigure}[b]{0.24\textwidth}
        \centering
        \includegraphics[width=\textwidth]{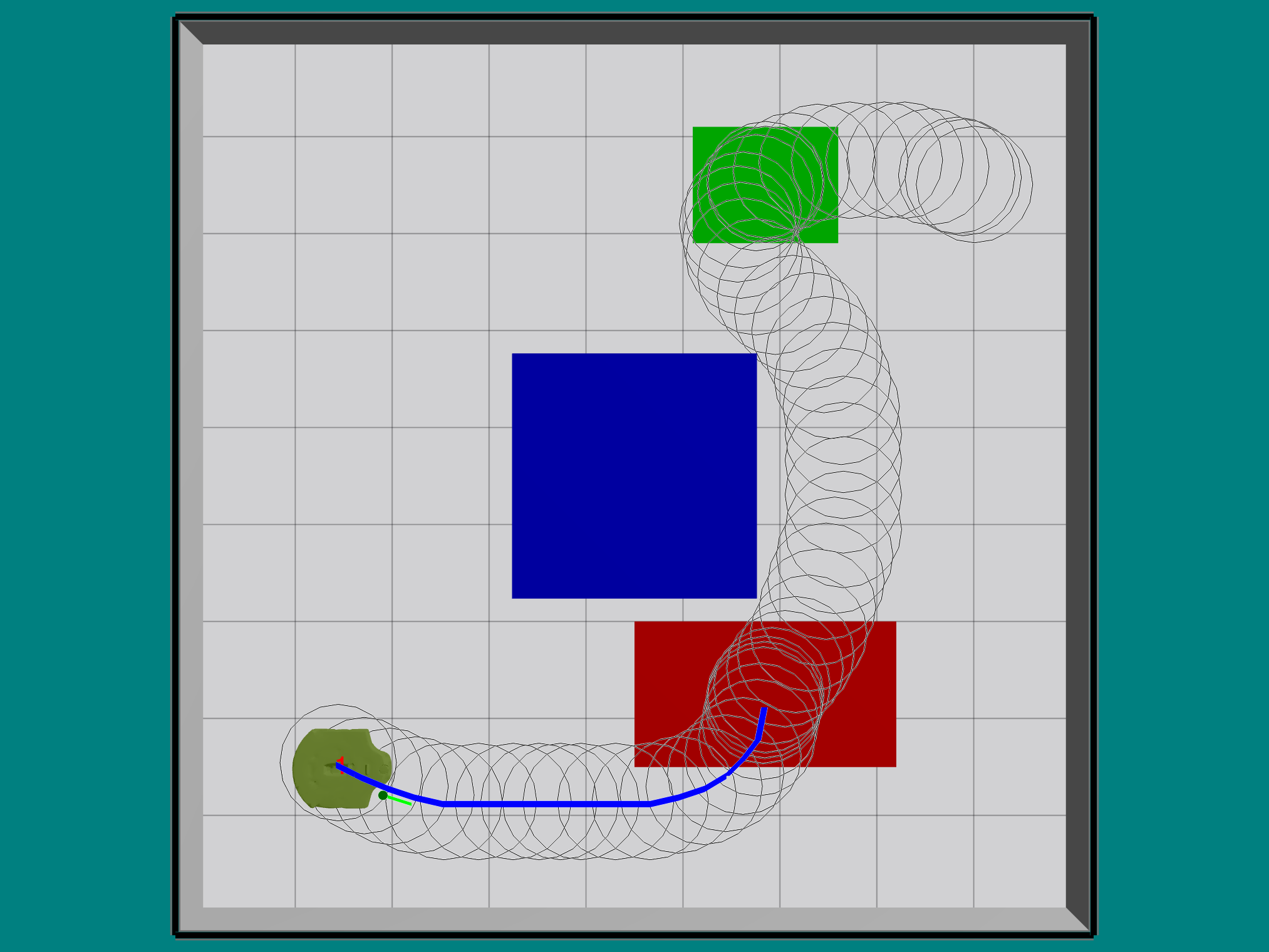}
        \caption{\texttt{stlcg}}
        \label{fig:stlcg}
    \end{subfigure}
    \begin{subfigure}[b]{0.24\textwidth}
        \centering
        \includegraphics[width=\textwidth]{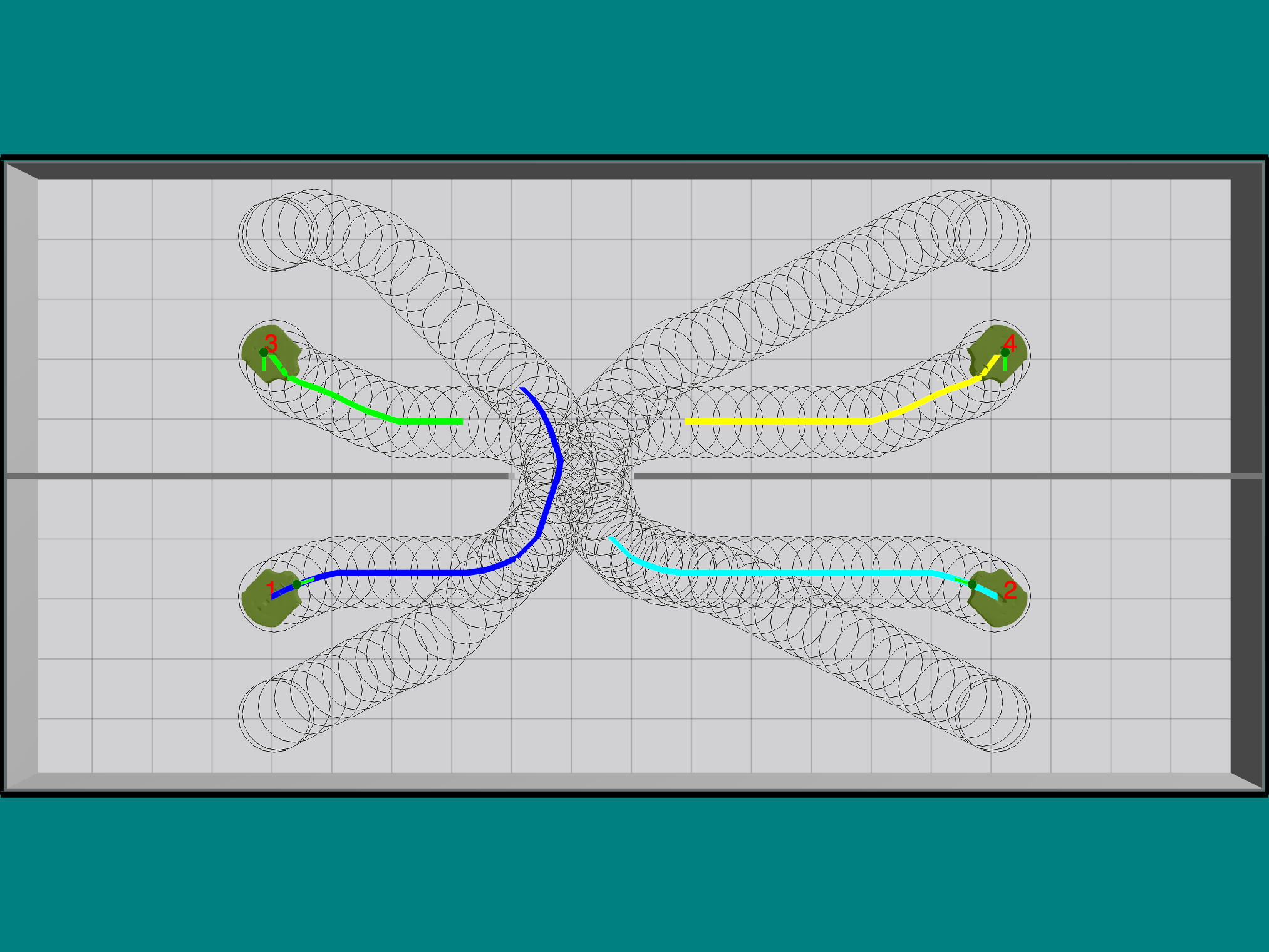}
        \caption{\texttt{door}}
        \label{fig:door}
    \end{subfigure}
    \begin{subfigure}[b]{0.24\textwidth}
        \centering
        \includegraphics[width=\textwidth]{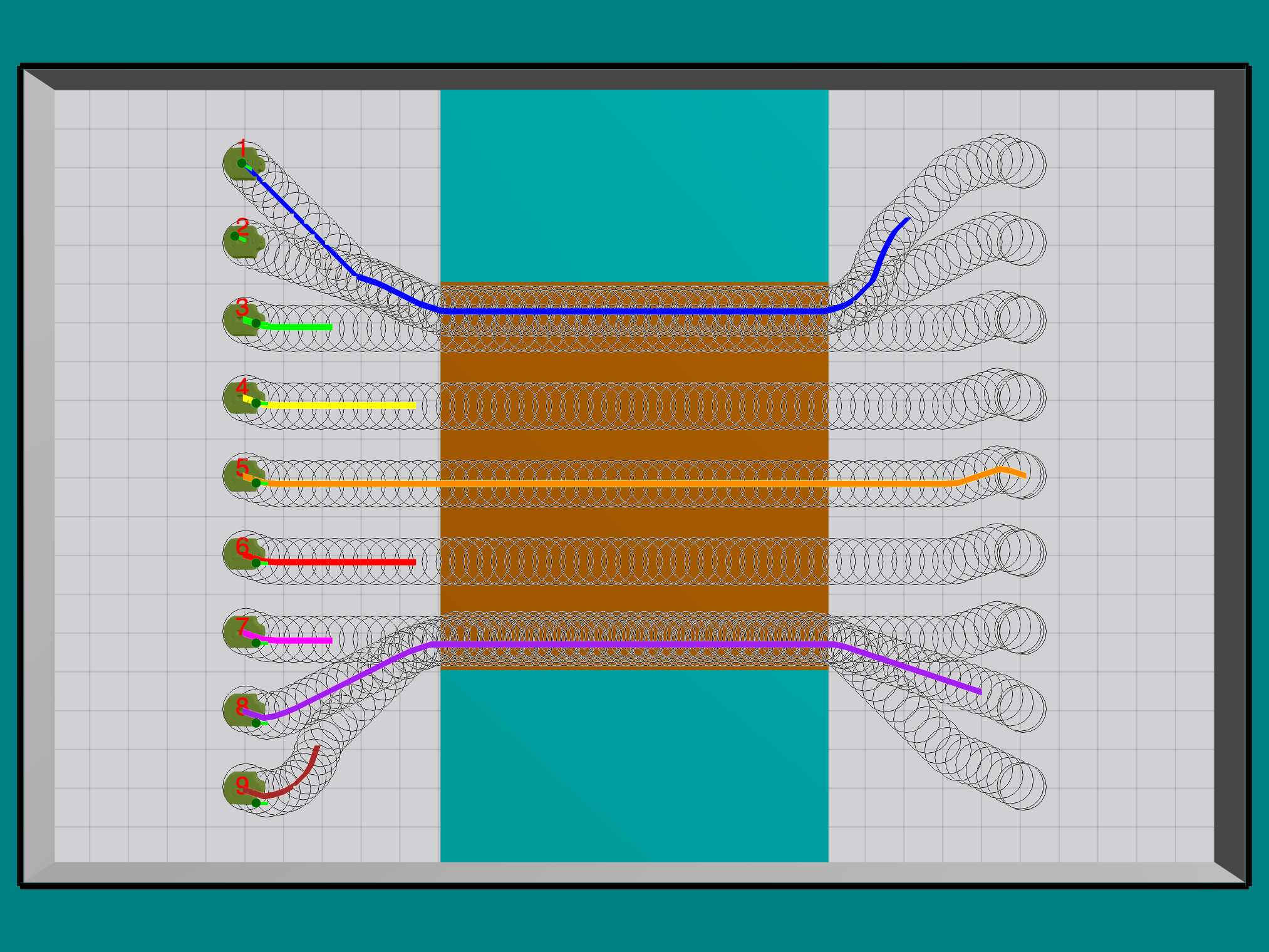}
        \caption{\texttt{bridge}}
        \label{fig:bridge}
    \end{subfigure}
    \begin{subfigure}[b]{0.24\textwidth}
        \centering
        \includegraphics[width=\textwidth]{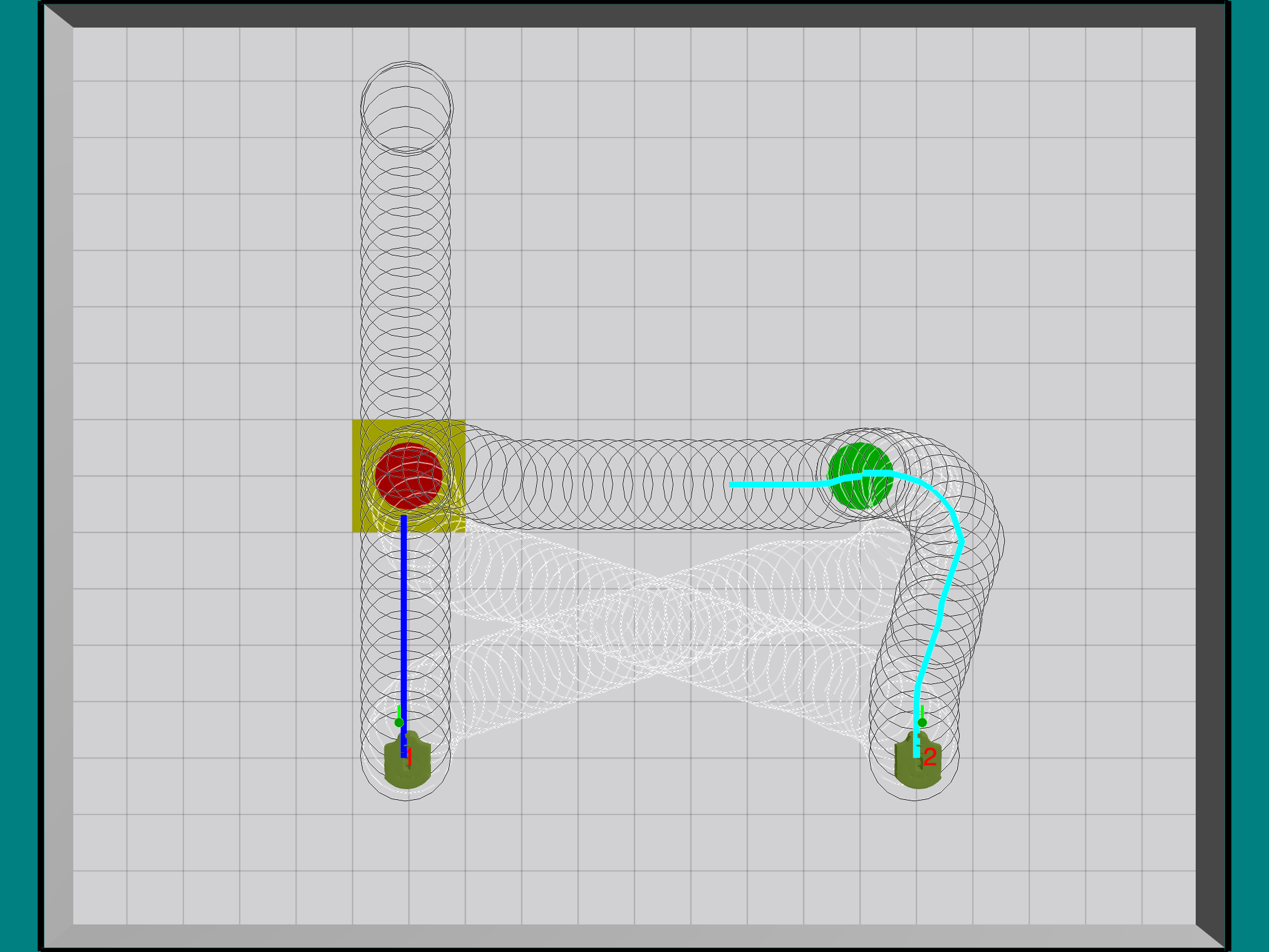}
        \caption{\texttt{cart}}
        \label{fig:cart}
    \end{subfigure}
    
    \begin{subfigure}[b]{0.24\textwidth}
        \centering
        \includegraphics[width=\textwidth]{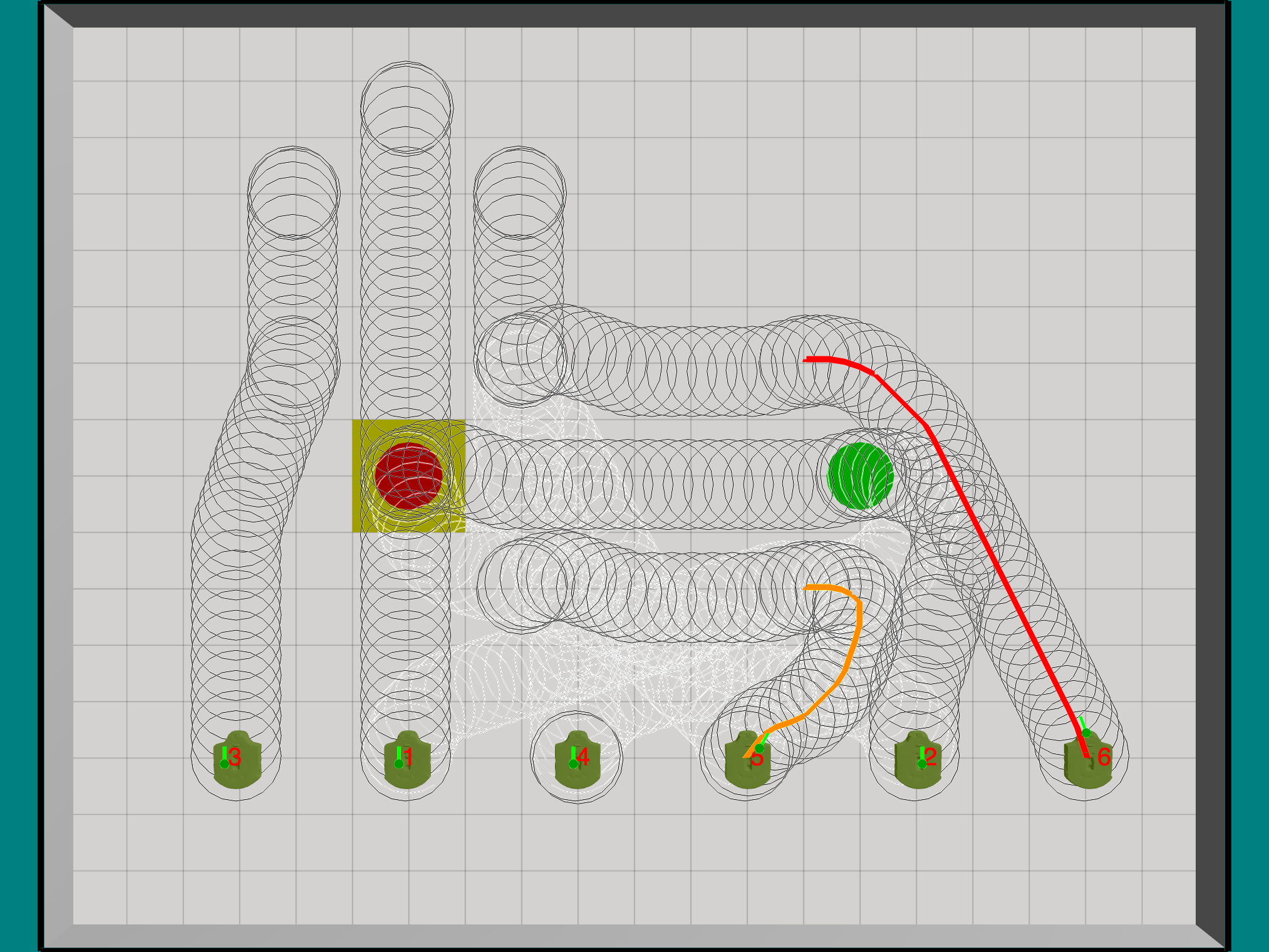}
        \caption{\texttt{escort}}
        \label{fig:escort}
    \end{subfigure}
    \begin{subfigure}[b]{0.485\textwidth}
        \centering
        \includegraphics[width=\textwidth]{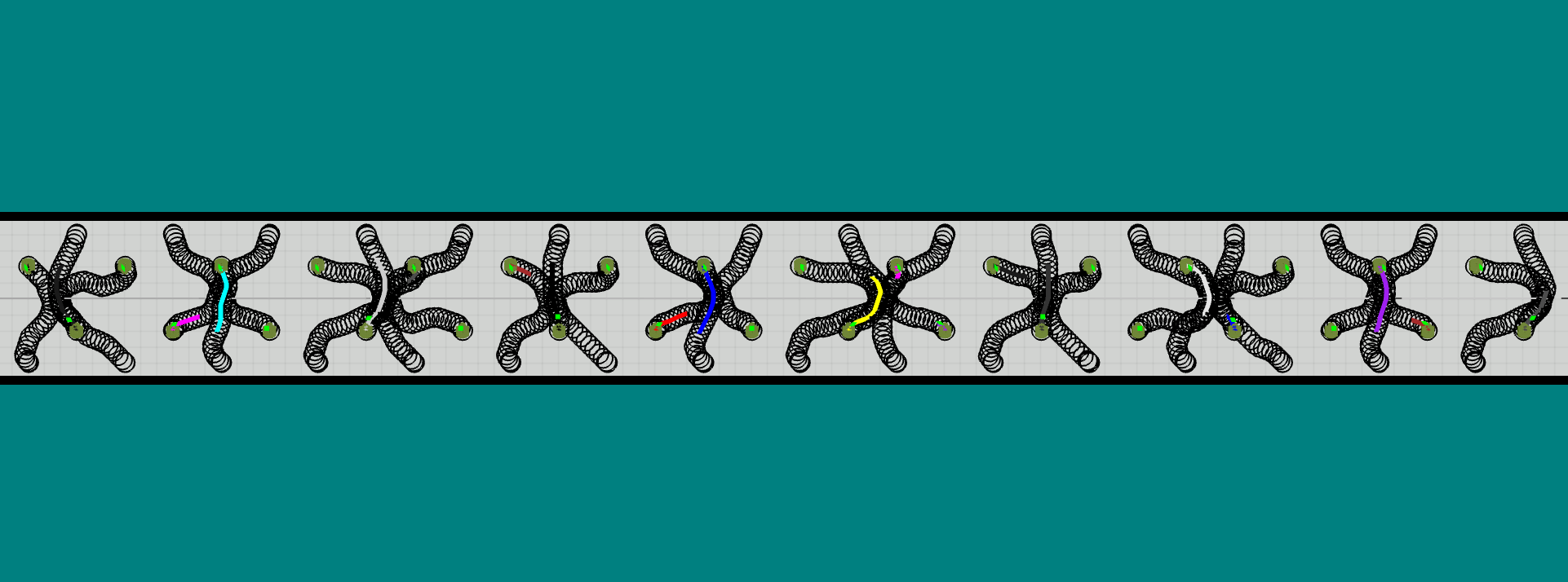}
        \caption{\texttt{door} (32 robots)}
        \label{fig:door32}
    \end{subfigure}
    \begin{subfigure}[b]{0.24\textwidth}
        \centering
        \includegraphics[width=\textwidth]{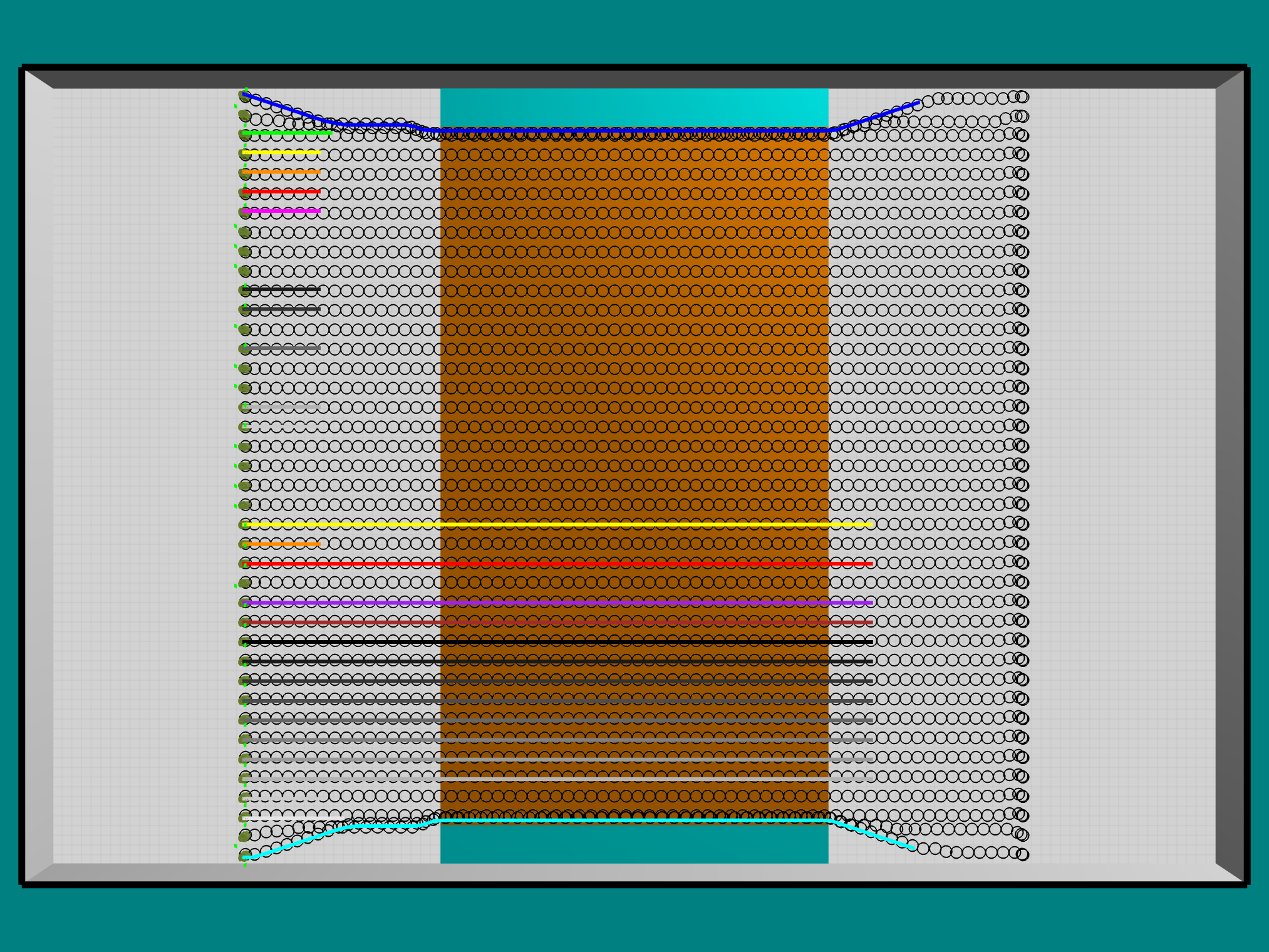}
        \caption{\texttt{bridge} (40 robots)}
        \label{fig:bridge40}
    \end{subfigure}
    \caption{\small Experiment results. Black and white circles are waypoints for the selected and unselected paths. Solid lines are the paths between the first two TSJPs. In (d), (e), the red and green circles are the full and empty carts, respectively. Videos of the experiments can be found at \url{https://www.amazon.science/scalable-multi-robot-task-allocation-and-coordination-under-signal-temporal-logic-specifications}.}
    \label{fig:experiments}
    \vspace{-8pt}
\end{figure*}
\paragraph{Single robot} We first test the algorithm on a single-robot single-path scenario (Fig.~\ref{fig:stlcg}) introduced in \cite{leung2023backpropagation}, referred to as \texttt{stlcg}, where the robot needs to visit and stay in the red region for $20s$, then visit and stay in the green region for $20s$, and always avoid the blue region. Let the intervals on the path corresponding to the red and green regions be $I_r$ and $I_g$. The specification is expressed using RP-STL as $F_{[0,T]}G_{[0,20]}(\sigma_1\in I_r) \land  F_{[0,T]}G_{[0,20]}(\sigma_1\in I_g)$. Since always avoiding the blue region can be enforced by the motion planner, we omit it in the formula. 

\paragraph{Interference} We then test the algorithm for interference constraints using a similar example as in \cite{sun2022multi} (Fig.~\ref{fig:door}), referred to as \texttt{door}, where $4$ robots need to pass through a narrow door to reach the other side of the map. 
To compare with~\cite{sun2022multi}, we assume each robot has only one reference path. The RP-STL specification is the conjunction of all interference constraints for all critical sections, see~\eqref{eq:interfere}. 

\paragraph{Counting formula} We extend the scenario in Example~\ref{ex:bridge} (Fig.~\ref{fig:bridge}), referred to as \texttt{bridge}, 
to include $9$ robots, each with one reference path. The specification  requires $\leq 3$ robots can be on the bridge simultaneously, similar to~\eqref{eq:example}. 

\paragraph{Task-specific requirements} 
We test our approach in a 
warehouse scenario with task-specific requirements. Here, we have work stations where packages are loaded into carts. In a first scenario (\texttt{cart}), once a cart is full, one robot must transport it to a truck for delivery, while a second robot must replenish the station with an empty cart.  Each robot is assigned two paths $p_i^1$ and $p_i^2$, corresponding to transporting the full or empty cart, respectively. Task assignment is open in the specification, and is computed as part of the MILP solution. 
The specification states that the robot returning the empty cart cannot proceed to the station until the other robot has removed the full cart; and that the station cannot be empty for more than $20s$, in order to avoid congestion caused by incoming packages. Let $[l_i^j,u_i^j]$ be the interval on the path $p_i^j$ where the robot is at the station (manipulating the cart). The specification is $ \varphi_{int} \land (\varphi_{swap}^{12} \land \varphi_{swap}^{21})$ where $\varphi_{int}$ is the conjunction of all interference constraints, and
\begin{equation*}
\begin{aligned}
    &\varphi_{swap}^{ii'} = (\sigma_{i'} < [M, l_{i'}^2]\cdot \mathbf z_{i'}^\top)\ U_{[0,T]}\ (\sigma_{i} \geq [u_i^1, -M]\cdot \mathbf z_i^\top) \land \\
    & G_{[0,T]}(\sigma_i\geq [l_i^1, -M]\cdot \mathbf z_i^\top \Rightarrow F_{[0,20]}\ \sigma_{i'}\geq [-M, u_{i'}^2]\cdot \mathbf z_{i'}^\top).
\end{aligned}    
\end{equation*}

We then increase the complexity of the scenario by introducing an escort task, referred to as \texttt{escort}. Since a robot may have limited visibility while carrying a cart, we require that at least two additional robots accompany it when it is carrying the cart. We introduce $4$ more robots $r_3$, $r_4$, $r_5$, $r_6$ for the escorting task. Each robot is assigned multiple paths corresponding to which cart it escorts. 
The RP-STL specification is $\varphi_{int}\land(\varphi_{swap}^{12} \land \varphi_{swap}^{21})\land\varphi_{esc}$, where $\varphi_{esc}$ (omitted for brevity, see video attachment) requires that at least two unladen robots are alongside each cart-laden robot.

\subsection{Comparison with other methods}
We compare our method with the piece-wise linear (PWL) path method~\cite{sun2022multi} and gradient-based methods. For \texttt{stlcg} and \texttt{door}, we compare with an enhanced gradient-based method from \cite{pant2017smooth} with gradient computed analytically using STLCG~\cite{leung2023backpropagation}. In other scenarios, the PWL method is not applicable because the STL syntax in \cite{sun2022multi} does not support temporal operators involving multiple robots. For these cases, we compare with the gradient-based method from~\cite{liu2023robust}, which also employs the counting semantics. However, this method cannot avoid inter-agent collision, so we simplify the problem for~\cite{liu2023robust} by omitting these constraints (which are not omitted in our approach). For all gradient-based methods, we set the time horizon $T=50$ and the time step $\Delta t = 1$. We did not compare with MILP-based MPC methods, e.g., \cite{raman2014model}, as the authors of~\cite{sun2022multi} have shown that their method outperforms~\cite{raman2014model} in these scenarios. The comparison results in terms of runtime are given in Table~\ref{tb:compare}. Note that for the PWL method, we only show the results of using sum of travel time (STT) as the objective function, as it is faster than using makespan. For gradient-based methods, the objective is to maximize the robustness and minimize the energy cost. 

For all the scenarios above, our approach finds the optimal solution, with the local controller tracking the TSJPs on time. The planned paths and the first two TSJPs for each scenario (using makespan) are shown in Fig.~\ref{fig:experiments}. Videos can be found in the supplementary material. We can see that our approach successfully satisfies constraints involving reach and avoid (\texttt{stlcg}), interference (\texttt{door}), density (\texttt{bridge}), and real-world requirements with complex STL specification and task allocation (\texttt{cart} and \texttt{escort}). It significantly outperforms other methods in terms of runtime for all the multi-robot scenarios, with minimal compromise on solution quality. The runtime of our method consists of two parts: the preprocessing phase (motion planning, interference computation, MILP encoding, etc) and the MILP solving phase. Although the motion planner is currently run sequentially for each path, this step could be greatly accelerated by parallel computing. In the \texttt{stlcg} case, the majority of the time is spent on preprocessing, with the MILP solving taking only $0.057s$.

\begin{table}
\centering
\begin{tabular}{|c|c|c|c|c|}
\hline
  Scenarios   &  Ours (MS) & Ours (STT) & PWL (STT) & Gradient \\
  \hline
  \texttt{stlcg}    & 0.612 & N/A & \textbf{0.243} & 1.211 \\
  \hline
  \texttt{door}      & \textbf{2.087} & 2.804 & 22.49 & fail \\
  \hline
  \texttt{bridge}     & \textbf{3.147} & 3.914 & N/A & 130.4 \\
  \hline
  \texttt{cart}      & \textbf{1.652} & 2.563 & N/A & 64.31\\
  \hline
  \texttt{escort}     & \textbf{32.57} & 247.6 & N/A & 269.6\\
  \hline
\end{tabular}
\caption{\small Runtime comparison (s). MS and STT refer to the use of makespan and sum of travel time as the objective. }
\vspace{-8pt}
\label{tb:compare}
\end{table}

\subsection{Scalability}
We further test the scalability of our approach using the \texttt{door} and \texttt{bridge} scenarios. In the \texttt{door} case, we increase the number of robots and doors proportionally (as shown in Fig.~\ref{fig:door32}) so that the number of critical sections increases linearly. In the \texttt{bridge} case, we increase the number of robots, the capacity of the bridge, and the width of the bridge proportionally (as shown in Fig.~\ref{fig:bridge40}) so that robots can always cross the bridge in $3$ waves while the number of critical sections is kept constant. Testing results are shown in Fig.~\ref{fig:scale}. The computation time for MILP does not increase significantly with the growth in the number of robots, especially in the \texttt{bridge} case which uses the counting constraints. Although adding robots will introduce more binary variables in the MILP, the additional constraints in these scenarios are well-structured and weakly coupled. So solvers such as Gurobi can exploit this structure to reduce computation time. The time for motion planning increases linearly with the number of robots but can be parallelized.

\begin{figure}
    \centering
    \begin{subfigure}[b]{0.23\textwidth}
        \centering
        \includegraphics[height=3.8cm]{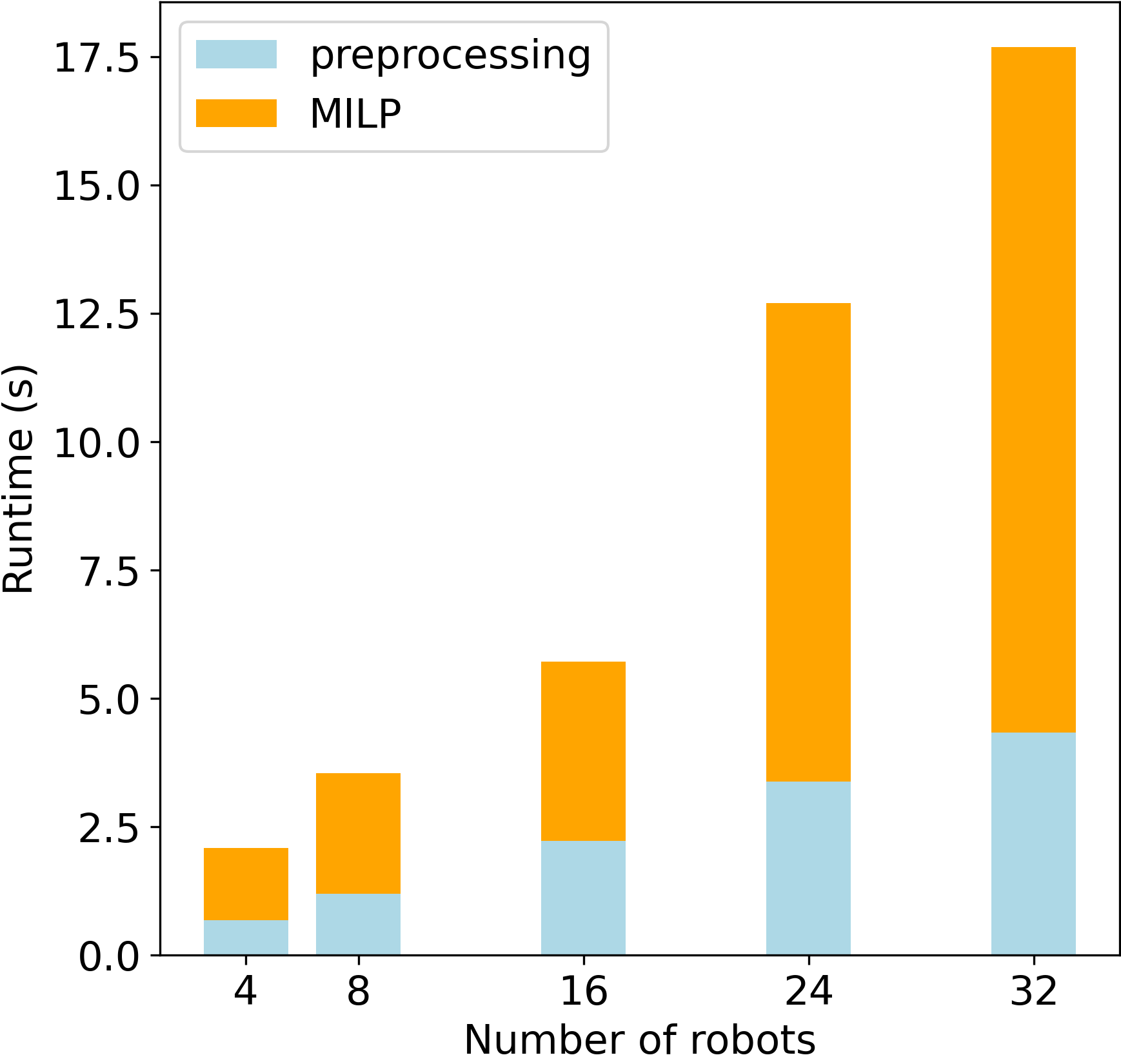}
        \caption{\texttt{door}}
        \label{fig:scale-door}
    \end{subfigure}
    \begin{subfigure}[b]{0.23\textwidth}
        \centering
        \includegraphics[height=3.8cm]{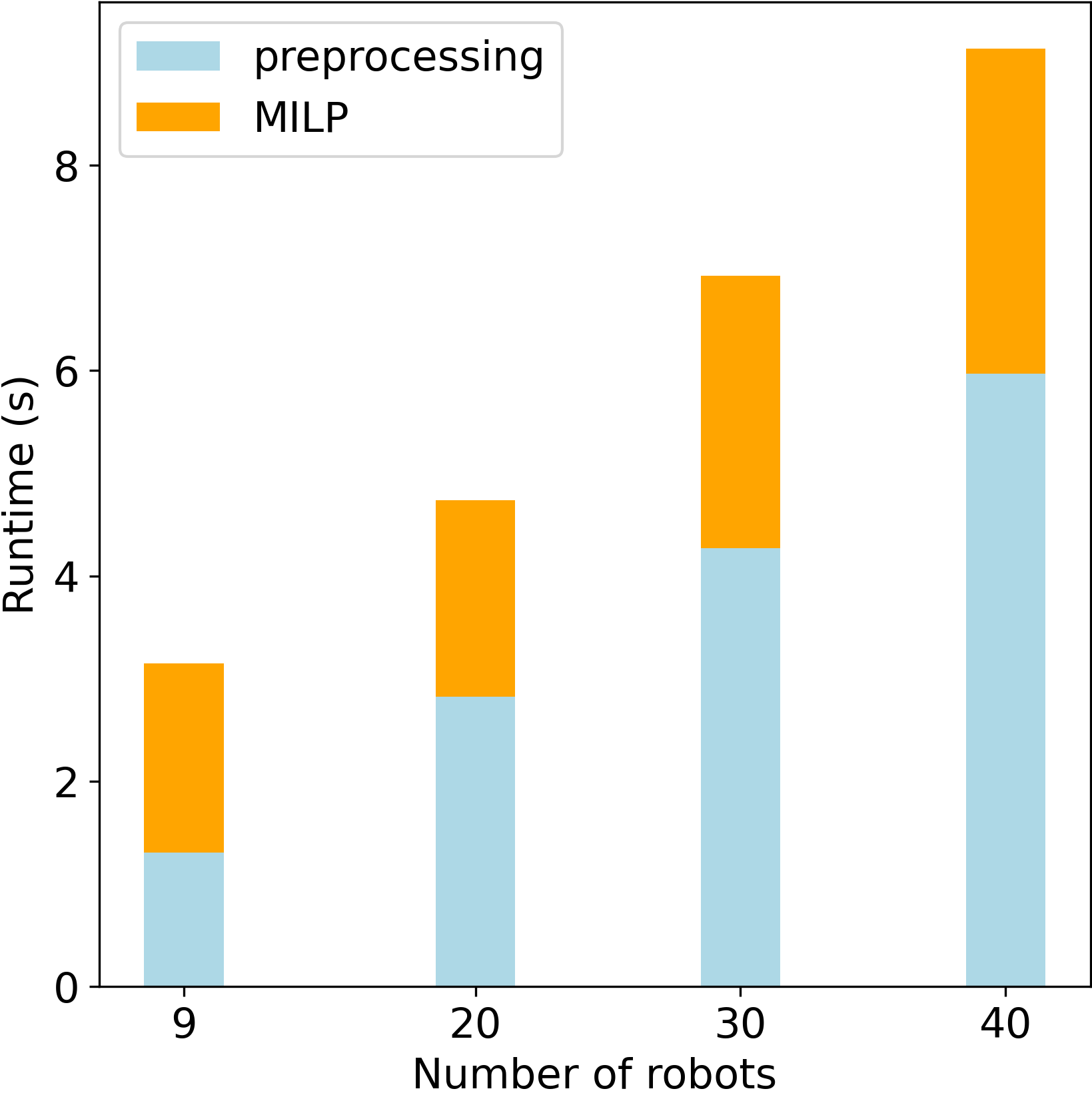}
        \caption{\texttt{bridge}}
        \label{fig:scale-bridge}
    \end{subfigure}
    \caption{\small Runtime for different number of robots.}
    \label{fig:scale}
    \vspace{-8pt}
\end{figure}

\section{Conclusion}
We propose an algorithm to operate a multi-robot system subject to STL specifications. Compared with other methods in the literature, our approach significantly 
reduces computational cost 
by decoupling of task allocation and coordination 
from motion planning and control 
while maintaining formal guarantees. Experimental results 
demonstrate efficiency 
and scalability to large robot teams and complex specifications.  





\bibliographystyle{IEEEtran}
\bibliography{references}

\end{document}